\documentclass[conference]{IEEEtran}
\IEEEoverridecommandlockouts
\usepackage{cite}
\usepackage{amsmath,amssymb,amsfonts}
\usepackage{algorithmic}
\usepackage[linesnumbered,ruled]{algorithm2e}
\usepackage{graphicx}
\usepackage{textcomp}
\usepackage{xcolor}
\usepackage{subfigure}
\usepackage{bm}
\usepackage{tabu}
\usepackage[capitalize]{cleveref}
\crefname{section}{Sec.}{Secs.}
\Crefname{section}{Section}{Sections}
\Crefname{table}{Table}{Tables}
\crefname{table}{Table}{Tables}

\newtheorem{theorem}{Theorem}

\newtheorem{remark}[theorem]{Remark}
\newtheorem{proof}{Proof}

\def\BibTeX{{\rm B\kern-.05em{\sc i\kern-.025em b}\kern-.08em
    T\kern-.1667em\lower.7ex\hbox{E}\kern-.125emX}}
\begin{document}

\title{Enhancing Adversarial Attacks: The Similar Target
Method\\
% \thanks{This research was supported by OPENATOM FOUNDATION}
}

\author{\IEEEauthorblockN{Shuo Zhang, Ziruo Wang, Zikai Zhou, Huanran Chen$^\dagger$\thanks{$^\dagger$corresponding author}}
\IEEEauthorblockA{
\textit{School of Computer Science,
        Beijing Institute of Technology}\\
\{huanranchen, shuozhangbit, ziruowang, zikaizhou\}@bit.edu.cn}
}

% \author{\IEEEauthorblockN{1\textsuperscript{st} Shuo Zhang}
% \IEEEauthorblockA{
% \textit{School of Computer Science,
%         Beijing Institute of Technology}\\
% Beijing, China \\
% shuozhangbit@bit.edu.cn}
% \and
% \IEEEauthorblockN{2\textsuperscript{nd} Ziruo Wang}
% \IEEEauthorblockA{
% \textit{School of Computer Science,
%         Beijing Institute of Technology}\\
% Beijing, China \\
% ziruowang@bit.edu.cn}
% \and
% \IEEEauthorblockN{3\textsuperscript{rd} Zikai Zhou}
% \IEEEauthorblockA{
% \textit{School of Computer Science,
%         Beijing Institute of Technology}\\
% Beijing, China \\
% zikaizhou@bit.edu.cn}
% \and
% \IEEEauthorblockN{4\textsuperscript{th} Huanran Chen$^{*}$}
% \IEEEauthorblockA{
% \textit{School of Computer Science,
%         Beijing Institute of Technology}\\
% Beijing, China \\
% huanranchen@bit.edu.cn}
% }
% \author{\IEEEauthorblockN{Anonymous Authors}}
% \and
% \and

\maketitle

%-----------------------------------------------------------
\begin{abstract}
% This document is a model and instructions for \LaTeX.
% This and the IEEEtran.cls file define the components of your paper [title, text, heads, etc.]. *CRITICAL: Do Not Use Symbols, Special Characters, Footnotes, 
% or Math in Paper Title or Abstract.
% An intriguing property of adversarial examples is their strong transferability, which enables attackers to create adversarial examples using their own models and then apply to deployed models, posing a threat to the existing deep learning system and raising security concerns. Several methods have been proposed to enhance transferability, one of the most effective methods is ensemble attacks which has demonstrated its efficacy. However, prior approaches simply average logits, probabilities, or losses for model ensembling, lacking a comprehensive analysis of how and why model ensembling significantly improves transferability. 
% In this paper, to further exploit the information in each model, we propose to regularizes the optimization direction to simultaneously attack all surrogate models via promoting cosine similarity between the gradients.
% Experimental results on ImageNet validate the effectiveness of our approach in improving adversarial transferability. Our method outperforms state-of-the-art attackers on 18 discriminative classifiers and adversarially trained models.
Adversarial examples are notably characterized by their strong transferability, allowing attackers to craft these examples on their models and subsequently deploy them against other models. This poses significant threats to existing deep learning systems and raises substantial security concerns. While several methods have been developed to enhance transferability, ensemble attacks stand out for their effectiveness. However, previous approaches of ensemble attacks typically rely on simple averaging of logits, probabilities, or losses for ensembling, without delving into the underlying reasons for the improved transferability. In this work, we propose a new approach that makes full use of the information of each surrogate model by regularizing the optimization direction to concurrently attack all surrogate models. This is achieved by promoting cosine similarity between their gradients. Extensive experiments conducted on the ImageNet dataset demonstrate the superior efficacy of our method in enhancing adversarial transferability. Notably, our approach outperforms leading state-of-the-art attackers across 18 discriminative classifiers and adversarially trained models, underscoring its potential in this domain.
\end{abstract}


\section{Introduction}

Deep learning has experienced significant advancements in recent years. However, a notable vulnerability of these models lies in their susceptibility to adversarial attacks. These attacks involve introducing imperceptible perturbations to original images, consequently impairing the models' performance~\cite{carlini2017towards,goodfellow2014explaining, wei2023cfa, wei2022extracting, wei2024weighted, zhang2023using, chen2023robust}. Recent studies have demonstrated that it is possible to manipulate state-of-the-art models like OpenAI's GPT-4 or Google's Bard, inducing them to divulge private information or produce harmful content through specifically designed objective functions~\cite{zou2023universal,wei2023jailbreak,dong2023robust,piet2023jatmo}.

Furthermore, the property of transferability in adversarial examples~\cite{liu2016delving, dong2018boosting} poses an additional risk. This property enables attackers to create adversarial examples using their models and then apply them to compromise deployed models. Such transfer attacks require neither prior knowledge of the target models nor access to their code, thereby amplifying the threat to the practical deployment of deep learning models and broader social security. This issue is particularly pertinent in the context of state-of-the-art large vision-language models like LLaMA and GPT-4~\cite{dong2023robust,wei2023jailbreak}.

\begin{figure}[t]
    \centering
    \includegraphics[width=9cm]{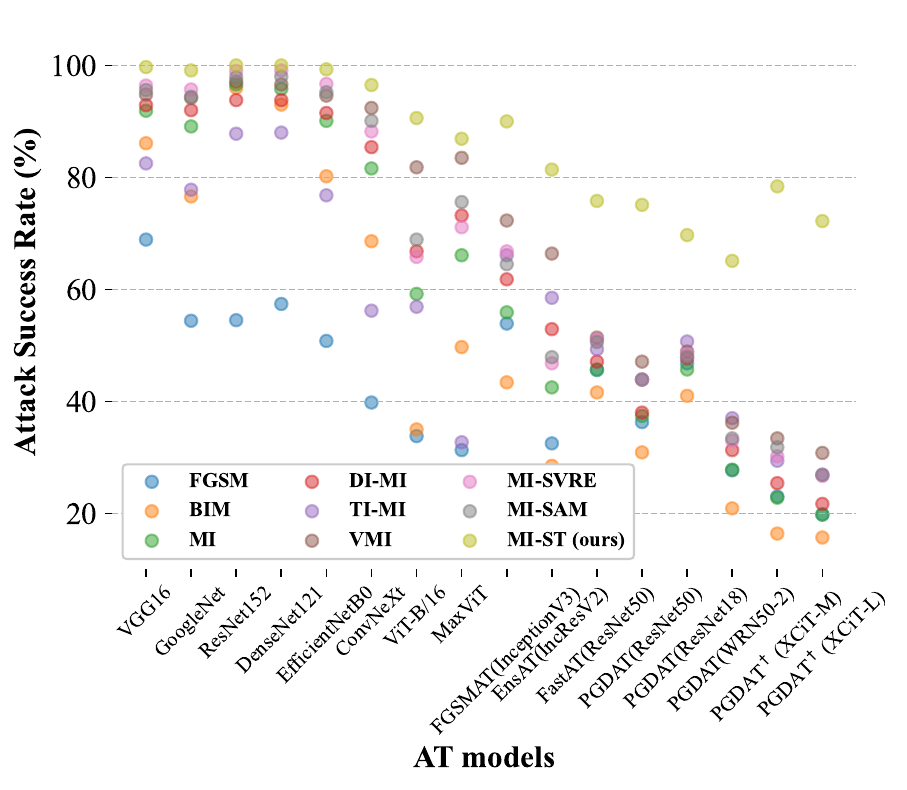}
    \caption{Black-box attack success rate (\%) on adversarially trained models by different algorithms. Our method outperforms the existing methods by a large margin. For more detail, refer to \cref{exp:core_result}.}
    \label{fig:1performance}
\end{figure}

Researchers have made every effort to study the transferability of adversarial examples to improve security and understand their mechanisms. Dong et al.~\cite{dong2018boosting} propose the Momentum Iterative (MI) method which introduces the idea of momentum to the optimization of adversarial examples, and found that it can help the attacker to converge into a more desirable local optimum which has better transferability. Dong et al.~\cite{dong2019evading} propose the Translation-Invariant (TI) method by optimizing a perturbation over an ensemble of translated images, and significantly improving the transferability, which sheds light on crafting adversarial examples across multi-images to create a smoother landscape. Similarly, Xie et al.~\cite{xie2019improving} also propose the Diverse Inputs (DI) method by applying random transformations to the input images at each iteration. From an optimization perspective, Wang et al.~\cite{wang2021enhancing} introduce the variance tuning (VMI) method. This technique enhances transferability by reducing the gradient variance and adjusting the current gradient about neighboring gradients. The VMI method aims to extend the adversarial attack's 'vision', preventing it from being overly myopic and focusing on immediate gradients, thereby facilitating the generation of more effective and transferable adversarial examples.

\begin{figure*}[t]
	\setlength{\abovecaptionskip}{-5pt}
	\setlength{\belowcaptionskip}{-10pt}
	\begin{minipage}{0.32\linewidth}
		\centering
		\includegraphics[width=\linewidth]{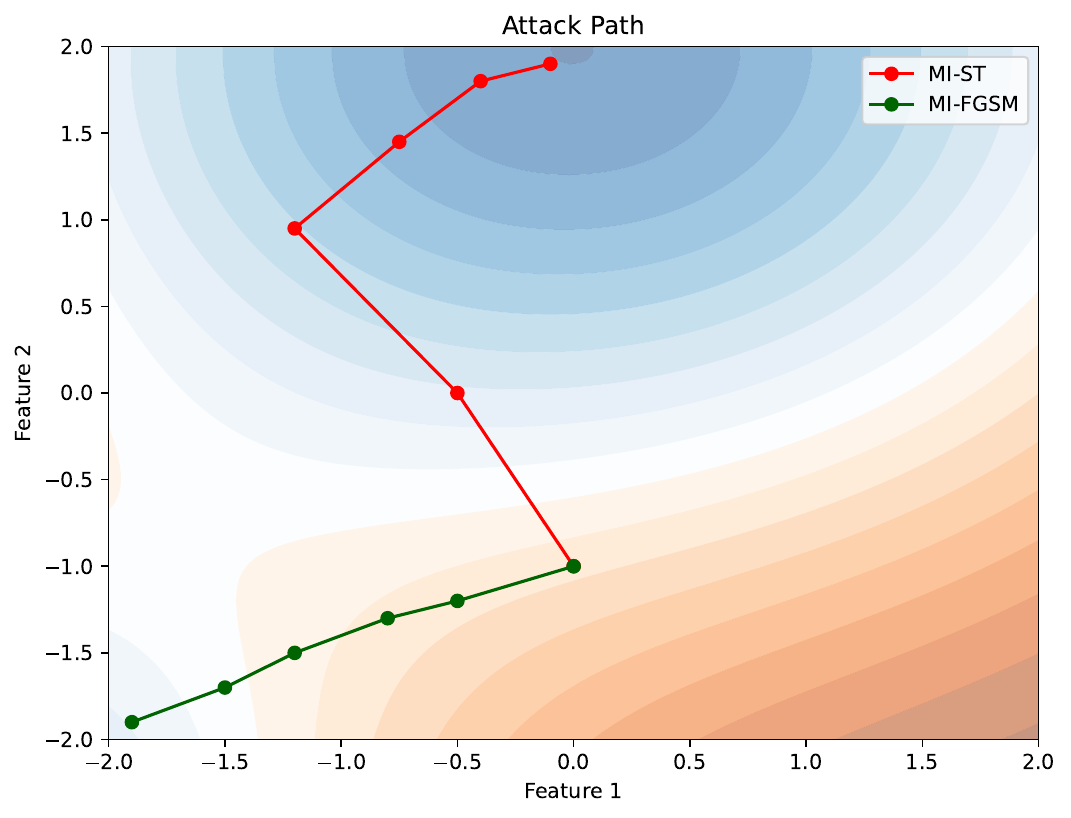}
		\caption{function $f$}
	\end{minipage}
	\hfill
	\begin{minipage}{0.32\linewidth}
		\centering
		\includegraphics[width=\linewidth]{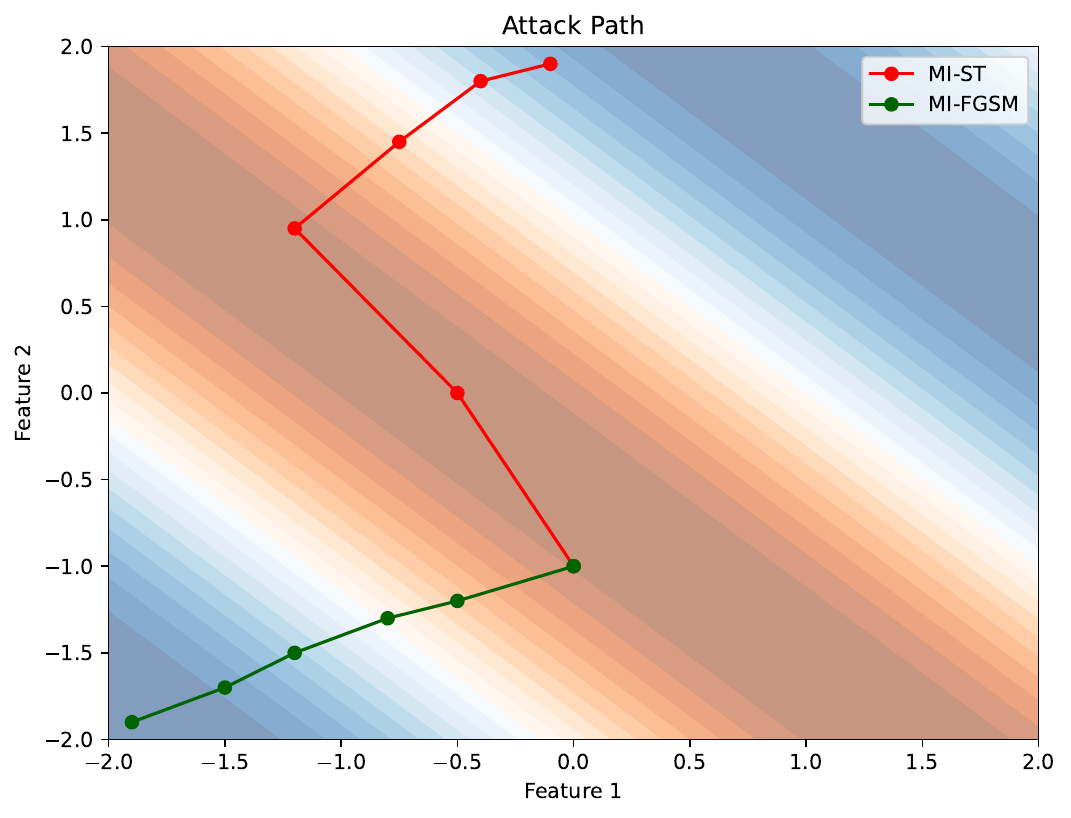}
		\caption{function $g$}
	\end{minipage}
	\hfill
	\begin{minipage}{0.32\linewidth}
		\centering
		\includegraphics[width=\linewidth]{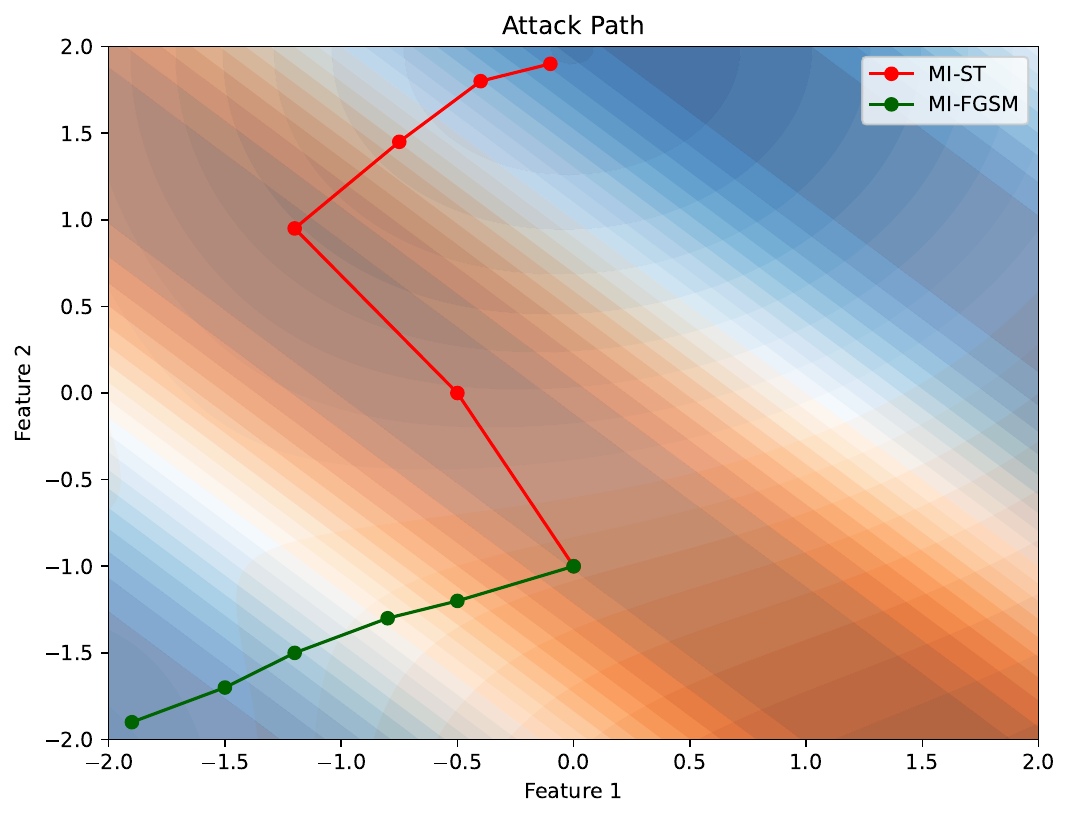}
		\caption{function $f+g$}
	\end{minipage}
 \vspace{0.5cm}
 \caption{Visualization of attack process. Our method initially directs towards regions where the cosine similarity between the gradients of functions \( f \) and \( g \) is high. This approach facilitates convergence to a common optimum for both \( f \) and \( g \), enabling our method to simultaneously attack these functions. In contrast, traditional methods primarily concentrate on optimizing \( f+g \). However, their resulting optimum often does not align with the optimum of \( f \) alone, leading to reduced effectiveness in attacking \( f \).}
 \label{fig:illustration}
\end{figure*}

Orthogonal to these studies, ensemble methods, which craft adversarial examples on multiple surrogate models simultaneously, have been proposed and proven effective, complementing earlier techniques~\cite{dong2018boosting}. However, in-depth research is scarce in this area. Traditionally, researchers have created ensemble models by averaging logits, probabilities, or losses of multiple surrogate models to craft adversarial examples, a process that often results in the loss of valuable information from individual models. In a novel approach,  Xiong et al.~\cite{xiong2022stochastic} introduce the SVRG optimizer into an ensemble attack to reduce the variance of different surrogate models' gradients during optimization, greatly enhancing the transferability. Recently, Chen et al.~\cite{chen2023rethinking} propose to encourage the cosine similarity between gradients of each surrogate model. This approach is theoretically grounded as it targets the upper bound of the distance between local optima, thereby enhancing generalization ability, i.e., transferability.  By promoting the cosine similarity between gradients, this method encourages adversarial examples to converge to a minimum that is common across all surrogate models. Consequently, it simultaneously attacks all the surrogate models rather than just the ensemble model, making it more likely for a new model to be susceptible to these adversarial examples.

Nonetheless, the algorithms presented in Chen et al.~\cite{chen2023rethinking} exhibit limitations. Primarily, their effectiveness is constrained due to the necessity to balance the promotion of cosine similarity between gradients and loss at a fixed ratio of $1: \frac{\beta}{2}$. Besides, the interplay between loss and cosine similarity, corresponding to optimization and regularization respectively, remains inadequately explored. Moreover, the approximation error of their algorithm is $O(\beta^3)$, which compromises their ability to effectively enhance cosine similarity, potentially diminishing their overall efficacy.

In this work, we propose the MI-ST algorithm to circumvent the constraints of weight balancing and delve deeper into the interplay between optimization and regularization. We meticulously develop a novel algorithm, termed the Similar Target method (ST), which effectively encourages the cosine similarity between gradients while disentangling it from the traditional loss and regularization terms. Remarkably, our method can be theoretically demonstrated to possess a substantially lower approximation error compared to Chen et al.'s approach~\cite{chen2023rethinking}. This implies a more precise maximization of the cosine similarity between gradients. Additionally, we integrate our method with established state-of-the-art strategies to further amplify its effectiveness. Furthermore, comprehensive ablation studies have been conducted to assess the trade-offs between optimization and regularization, aiming to identify the optimal balance for various models. Empirical evaluations on the ImageNet dataset substantiate the efficacy of our approach in enhancing adversarial transferability. Overall, our method surpasses current state-of-the-art attackers across 18 leading classifiers and defense mechanisms, demonstrating its superiority in the field.

\section{Related Work}

\subsection{Adversarial attacks}

Formally, denote the original images as $\bm{x}^{real}$  and the adversarial examples as $\bm{x}^{adv}$. Let $\mathcal{F}$ represent the set of all image classifiers and  $\mathcal{F}_t\subset \mathcal{F}$ represent the set of surrogate models. Meanwhile, we use $f(\cdot)$ to denote the classifiers and the $L$ to denote the corresponding loss function~(e.g. cross-entropy loss). Crafting adversarial examples could be formalized as an optimization problem:
\begin{equation}
\label{eq:adv_defination}
    \arg\max_{\bm{x}^{adv}} \frac{1}{|\mathcal{F}|}\sum_{f \in \mathcal{F}}L(f(\bm{x}^{adv}),\bm{y}) \text{,  } s.t. \; ||\bm{x}^{adv} - \bm{x}^{real}||_{\infty}\leq \epsilon.
\end{equation}
This objective function means we need to find the adversarial sample that can maximize the objective function over all target models. In other words, we aim to find examples that all the target models perform poorly on. However, in real scenarios, attackers usually cannot access the deployed model in \( \mathcal{F} \). An alternative solution is to craft adversarial examples on surrogate models \( \mathcal{F}_t \) and transfer them to the target models, a.k.a. transfer attacks. In other words, we maximize the objective function on surrogate models:
\begin{equation}
    \arg\max_{\bm{x}^{adv}} \frac{1}{|\mathcal{F}_t|}\sum_{f \in \mathcal{F}_t}L(f(\bm{x}^{adv}),\bm{y}) \text{,  } s.t. \; ||\bm{x}^{adv} - \bm{x}^{real}||_{\infty}\leq \epsilon. \label{eq2}
\end{equation}

\subsection{Transfer attacks}
Due to its effectiveness and simplicity, transfer attacks have been garnering significant attention. Several methods have been devised to enhance transferability, and we classify them into the following three categories.

\textbf{Gradient-based Methods.} Drawing an analogy to the optimization and regularization in neural network training, Dong et al.~\cite{dong2018boosting} propose the Momentum Iterative (MI) method and Lin et al.~\cite{lin2019nesterov} propose the Nesterov Iterative (NI) method, which introduces the momentum and the Nesterov accelerated gradient to prevent the adversarial examples from falling into the undesired local optima, therefore enhance the generalization ability~\cite{chen2022bootstrap}. Wang et al.~\cite{wang2021enhancing} propose variance tuning (VMI) to reduce the variance of the gradient and tune the current gradient with the gradient variance in the neighborhood of the previous data point. It also extends the adversarial attack's "vision" by utilizing the neighborhood gradient, preventing it from being myopic, and solely focusing on an immediate gradient. We will show that our method can act as a plug-and-play regularization term and be incorporated with these methods in \cref{sec:incorporation}.

\textbf{Input Transformations.} Analogous to the data augmentation, Xie et al.~\cite{xie2019improving} propose the Diverse Inputs (DI) method by applying random transformations to the input images at each iteration. Dong et al.~\cite{dong2019evading} propose the Translation-Invariant (TI) method by optimizing a perturbation over an ensemble of translated images. Lin et al.~\cite{lin2019nesterov} also propose to average the gradient of scaled copies of the input images. These methods add some preprocessing operations before feeding the input data into the neural network to enhance the robustness of adversarial examples toward common perturbations and transformations, thus it is also orthogonal to our work.

\textbf{Ensemble Attacks.} Similar to ensemble learning, using an ensemble of surrogate models can significantly improve the transferability. However, there is a scarcity of research on ensemble attacks. Traditionally, researchers have created ensemble models by averaging the logits, probabilities, or losses of multiple surrogate models to craft adversarial examples, a process that often results in the loss of valuable information from individual models. Recently, Huang et al.~\cite{huang2023t} draw an analogy between the number of classifiers used in crafting adversarial examples and the size of training sets in neural network training, providing a new theoretical understanding that increasing the surrogate models could reduce the generalization error upper bound. On the other hand, while this is so, Xiong et al.~\cite{xiong2022stochastic} incorporate the SVRG optimizer into the ensemble attack to decrease the variance of gradients during optimization. However, this approach still does not fully utilize the potential of individual models and wastes a significant amount of information. 

Recently, Chen et al.~\cite{chen2023rethinking} propose the common weakness of the ensemble models by showing that both the flatness of loss landscape~\cite{chen2022bootstrap, wei2023sharpness} and the distance between the local optimums are strongly correlated with the transferability, and the cosine similarity between gradients are upper bound of the latter term. Intuitively, by promoting the cosine similarity between gradients, this method encourages adversarial examples that converge to a minimum that is common across all surrogate models, thus it simultaneously attacks all the surrogate models, making it more likely for a new model to be susceptible to these adversarial examples. However, the efficacy of algorithms in Chen et al.~\cite{chen2023rethinking} are limited. First, their algorithms are ineffective because their methods have to promote the cosine similarity between gradients and loss with a fixed ratio $1: \frac{\beta}{2}$. Besides, the interplay between loss and cosine similarity, which influence optimization and regularization, remains inadequately explored.

\section{Methodology}

\begin{algorithm}[t]
\textbf{Require:} natural image $\bm{x}^{real}$, label $\bm{y}$, loss function $L$, surrogate models $\mathcal{F}_t = \{f_i\}_i^n$, perturbation budget $\epsilon$, iterations $T$, learning rate $\beta$, loss weight $\lambda_1$, cosine weight $\lambda_2$.

\textbf{Initialize:} $\bm{x}_0 = \bm{x}^{real}$\;
\For{$t$ = $0$ : $T-1$}
{
\For{$i$ = $1$ : $n$}
    {Calculate $\bm{g}_{i-1} = \nabla_{\bm{x}}L(f_{i-1}(\bm{x}_t^{i-1}),\bm{y})$\;Update adversarial sample by $\bm{x}_t^i = clip_{\bm{x}^{real},\epsilon}(\bm{x}_t^{i-1} + \beta \cdot \frac{\bm{g}_{i-1}}{||\bm{g}_{i-1}||_2})$}
    
    {Calculate $\bm{g}_{mean} = \frac{1}{|\mathcal{F}_t|}\sum_i \frac{\nabla_{\bm{x}}L(f_i(\bm{x}_t),\bm{y})}{\|\nabla_{\bm{x}}L(f_i(\bm{x}_t),\bm{y})\|_2}$\;
    $\bm{g}_{cos} = \bm{x}_t^{new} - \bm{x}_t - \beta \bm{g}_{mean}$\;
    Calculate $\bm{g}_{whole} = \lambda_1 \bm{g}_{mean} + \frac{2\lambda_2}{\beta^2} \bm{g}_{cos}$\;
    Update $\bm{x}_{t+1} = clip_{\bm{x}^{real},\epsilon}(\bm{x}_t + \bm{g}_{whole})$}
}
\textbf{Return:} ${\bm{x}}_T$
\caption{ST attacker}
\label{algo}
\end{algorithm}

To this end, by an insightful and complicated mathematical derivation, we propose a new algorithm called MI-ST, which can calculate the exact derivative of the cosine similarity between gradients, enabling us to trade-off between the gradient of the original loss (optimization) and the gradient of cosine similarity (regularization). This section is organized as follows: We illustrate our algorithm in \cref{algo} and \cref{figure:algorithm} in \cref{sec:description}. We also provide our insightful mathematical derivation in \cref{sec:derivation}. Finally, we demonstrate how our algorithm can be combined with state-of-the-art algorithms in \cref{sec:incorporation}.

\subsection{Our algorithm}
\label{sec:description}

\begin{table*}[t]
\centering
\caption{\textbf{Black-box attack success rate($\%$,$\uparrow$) on NIPS2017 dataset}. Our method performs well on 16 normally trained models with various architectures.}
\label{tab:t1}
\begin{tabu}{c|ccccccccccc} 
\hline
Method         & FGSM & BIM  & MI   & DI-MI & TI-MI & VMI  & MI-SVRG & MI-SAM & MI-CSE & MI-CWA & MI-ST           \\ 
\hline
AlexNet~\cite{alexnet}        & 76.4 & 54.9 & 73.2 & 78.9  & 78.0  & 83.3 & 82.5    & 81.0   & 93.6   & 94.6   & 94.2            \\
VGG16~\cite{vgg}          & 68.9 & 86.1 & 91.9 & 92.9  & 82.5  & 94.8 & 96.4    & 95.6   & 99.6   & 99.5   & \textbf{99.7}   \\
GoogleNet~\cite{googlenet}      & 54.4 & 76.6 & 89.1 & 92.0  & 77.8  & 94.2 & 95.7    & 94.4   & 98.8   & 99.0   & \textbf{99.1}   \\
InceptionV3~\cite{inception}    & 54.5 & 64.9 & 84.6 & 89.0  & 75.7  & 91.1 & 92.6    & 89.2   & 97.3   & 97.2   & 97.1            \\
ResNet152~\cite{resnet}      & 54.5 & 96.0 & 96.6 & 93.8  & 87.8  & 97.1 & 99.0    & 97.9   & 99.9   & 99.8   & \textbf{100.0}  \\
DenseNet121~\cite{densenet}    & 57.4 & 93.0 & 95.8 & 93.8  & 88.0  & 96.6 & 99.1    & 98.0   & 99.9   & 99.8   & \textbf{100.0}  \\
SqueezeNet~\cite{iandola2016squeezenet}     & 85.0 & 80.4 & 89.4 & 92.9  & 85.8  & 94.2 & 96.1    & 94.1   & 99.1   & 99.3   & 99.0            \\
ShuffleNetV2~\cite{ma2018shufflenet}   & 81.2 & 65.3 & 79.9 & 85.7  & 78.2  & 89.9 & 90.3    & 87.9   & 97.2   & 97.3   & 96.9            \\
MobileNetV3~\cite{mobilenet}    & 58.9 & 55.6 & 71.8 & 78.6  & 74.5  & 87.3 & 80.6    & 80.7   & 94.6   & 95.7   & 95.2            \\
EfficientNetB0~\cite{tan2019efficientnet} & 50.8 & 80.2 & 90.1 & 91.5  & 76.8  & 94.6 & 96.7    & 95.2   & 98.8   & 98.9   & \textbf{99.3}   \\
MNasNet~\cite{tan2019mnasnet}        & 64.1 & 80.8 & 88.8 & 91.5  & 75.5  & 94.1 & 94.2    & 94.3   & 99.1   & 98.7   & 98.9            \\
RegNetX400MF~\cite{regnet}   & 57.1 & 81.1 & 89.3 & 91.2  & 82.4  & 95.3 & 95.4    & 93.9   & 98.9   & 99.4   & 98.9            \\
ConvNeXt~\cite{liu2022convnext}       & 39.8 & 68.6 & 81.6 & 85.4  & 56.2  & 92.4 & 88.2    & 90.1   & 96.2   & 95.4   & \textbf{96.5}   \\
ViT-B/16~\cite{vit}       & 33.8 & 35.0 & 59.2 & 66.8  & 56.9  & 81.8 & 65.8    & 68.9   & 89.6   & 89.6   & \textbf{90.6}   \\
Swin-S~\cite{liu2021swin}         & 34.0 & 48.2 & 66.0 & 74.2  & 40.9  & 84.2 & 73.4    & 75.1   & 88.6   & 87.6   & 88.5            \\
MaxViT~\cite{tu2022maxvit}         & 31.3 & 49.7 & 66.1 & 73.2  & 32.7  & 83.5 & 71.1    & 75.6   & 85.8   & 85.9   & \textbf{86.9}   \\
\hline
\end{tabu}
\end{table*}

% \begin{equation}
% a+b=\gamma\label{eq}
% \end{equation}

As demonstrated in \cref{algo} and \cref{figure:algorithm}, for given natural images $\bm{x}^{real}$, we first iteratively perform gradient update using the normalized gradient obtained from the i-th model:

\begin{equation}
    \bm{x}_t^i = clip_{\bm{x}^{real},\epsilon}(\bm{x}_t^{i-1} + \beta \cdot \frac{\bm{g}_{i-1}}{||\bm{g}_{i-1}||_2}),\label{eq3}
\end{equation}

where $\bm{g}_i = \nabla_{\bm{x}} L(f_i({\bm{x}}), \bm{y})$, $clip_{\bm{x}^{real},\epsilon}$ operation is to clip the model to the $\epsilon$ neighborhood of the natural image $\bm{x}^{real}$ to control the perturbation budget. After this iterative update, we calculate the gradient of the original loss and our regularization term:

\begin{equation}
    \begin{aligned}
        &\bm{g}_{mean} = \frac{1}{|\mathcal{F}_t|}\sum_i^n \frac{\nabla_{\bm{x}}L(f_i({\bm{x}}_t),{\bm{y}})}{\|\nabla_{{\bm{x}}}L(f_i({\bm{x}}_t),{\bm{y}})\|_2}, \\
        &\bm{g}_{cos} = {\bm{x}}_t^{new} - {\bm{x}}_t - \beta \bm{g}_{mean}.\label{eq4}
    \end{aligned}
\end{equation}

Finally, we calculate the update and trade-off of the optimization and regularization by $\lambda_1$ and $\lambda_2$:
\begin{equation}
    \begin{aligned}
        \bm{g}_{whole} = \lambda_1 \bm{g}_{mean} + \frac{2\lambda_2}{\beta^2} \bm{g}_{cos}, \\
        \bm{x}_{t+1} = clip_{\bm{x}^{real},\epsilon}(\bm{x}_t + \bm{g}_{whole}).\label{eq5}
    \end{aligned}
\end{equation}

Following \cite{dong2019evading, chen2023rethinking, lin2019nesterov, xiong2022stochastic}, we update the adversarial example by the gradient that integrated with momentum to stabilize the update direction and avoid undesirable local optimum, as shown in \cref{figure0b}, where the final direction is the vector sum of $\bm{g}_{whole}$ and momentum. We repeat the above process until convergence or exceeding the iteration time limit.

% In MI-ST algorithm, after the initialization, the gradient of each surrogate model has been calculated in the inner loop iteration. We use $x^{i-1}$ and $g_{i-1}$ to update $x^i$. After finishing the whole inner loop iteration for one time, the whole update of the $x$ has been got and the mean of the gradient of the surrogate models $g_{mean}$ has been calculated. Then the cosine gradient $g_{cos}$ can be obtained using the whole update $x_t^{new} - x_t$ minus $g_{mean}$. Next, we can combine the $g_{mean}$ and $g_{cos}$ with two arbitrary weights, which now can be adjusted freely, and use it to update the $x$ in the outer loop iteration.

\begin{figure}[t]
\centering
\subfigure[MI]{
\includegraphics[width=3.75cm]{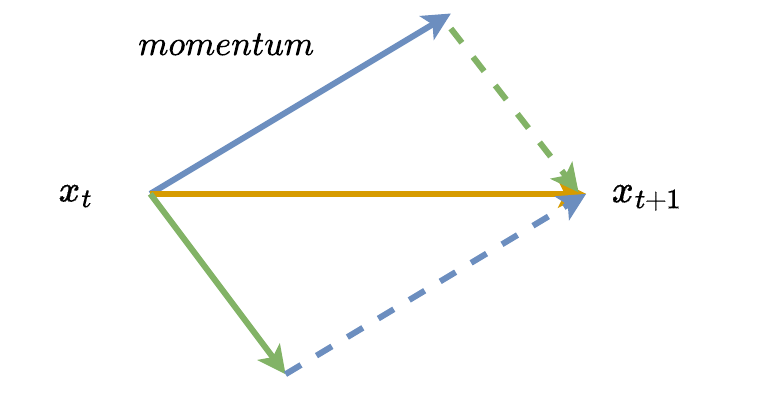}
\label{figure0a}
}
\quad
\subfigure[MI-ST]{
\includegraphics[width=3.78cm]{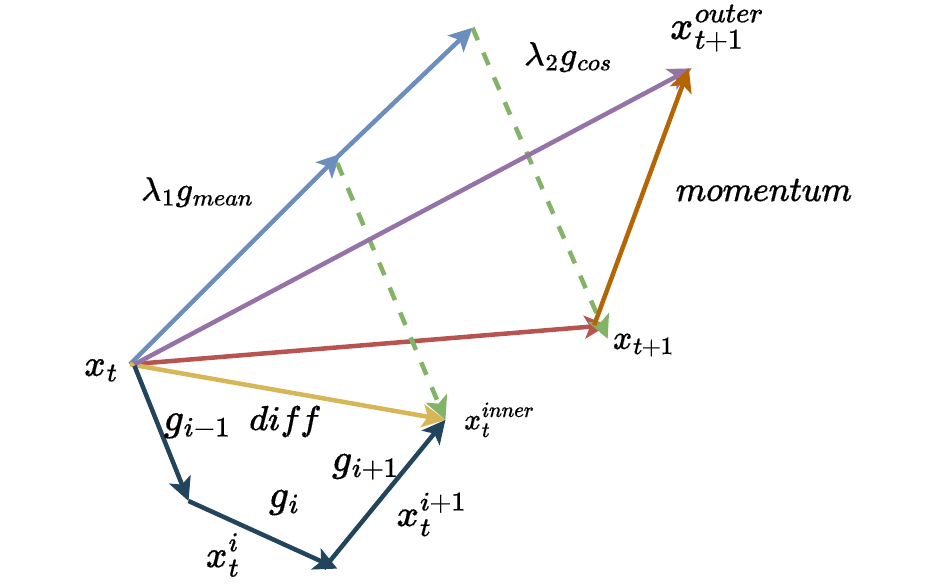}
\label{figure0b}
}
\caption{Illustration of our method.}
\label{figure:algorithm}
\vspace{-1ex}
\end{figure}

\subsection{The mathematical derivation}

In this subsection, we provide our derivation of proposed \cref{algo}.

\label{sec:derivation}

\begin{theorem}
    When $\beta \to 0$, Updating by our \cref{algo} is equivalent to optimizing:
    \begin{equation}
        \max_{\bm{x}} \lambda_1 \frac{1}{|\mathcal{F}_t|}\sum_{f \in \mathcal{F}_t}L(f_i({\bm{x}}), {\bm{y}}) + \lambda_2 \sum_{i,j}^{i<j< |\mathcal{F}_t|}\frac{\bm{g}_i^T\bm{g}_j}{\|\bm{g}_i\|\|\bm{g}_j\|}\label{eq6}
    \end{equation}
    where $\bm{g}_i = \nabla_{\bm{x}} L(f_i({\bm{x}}), {\bm{y}})$, $\lambda_1$ and $\lambda_2$ is the trade-off hyper-parameter setted in our algorithm.
\end{theorem}

\begin{proof}
Denote $\bm{g}_i'$ as the gradient at $i^{th}$ iteration in the inner loop, we can represent $\bm{g}_i'$ by $\bm{g}_i$ using Taylor expansion:
\begin{equation*}
    \begin{aligned}
    \bm{g}_i'&=\bm{g}_i+\bm{H}_i(\bm{x}_t^{i}-\bm{x}_t^0) \\
&=\bm{g}_i+\beta \bm{H}_i \sum_{j=1}^{i-1}\frac{ \bm{g}_j'}{\| \bm{g}_j'\|_2}  \\
&= \bm{g}_i+\beta \bm{H}_i \sum_{j=1}^{i-1}\frac{ \bm{g}_j+o(\beta)}{\| \bm{g}_j+o(\beta)\|_2} \\
&= \bm{g}_i+\beta \bm{H}_i \sum_{j=1}^{i-1}\frac{ \bm{g}_j}{\| \bm{g}_j\|_2} +O(\beta^2).\label{eq7}
    \end{aligned}
\end{equation*}

\begin{table*}[t]
    \begin{center}
      \caption{\textbf{Black-box attack success rate($\%$,$\uparrow$)}. Our method leads the performance on 8 adversarially trained models available on RobustBench. Note that PGDAT$^{\dag}$ is a variant of PGDAT tuned by bag of tricks. It turns out that our method improve the transferability of the adversarial examples.}
    \label{tab:t2}
    \scalebox{1}{
    \begin{tabular}{c|c|c|c|c|c|c|c|c}
    \hline
        Method & FGSMAT~\cite{adversarialMLAtScale} &  EnsAT~\cite{tramer2017ensemble} &  FastAT~\cite{wong2020fast} &  PGDAT~\cite{Engstrom2019Robustness} &  PGDAT~\cite{Engstrom2019Robustness} &  PGDAT~\cite{salman2020adversarially} &  PGDAT$^{\dag}$~\cite{debenedetti2022light} & PGDAT$^{\dag}$~\cite{debenedetti2022light} \\ \hline
        Backbone & InceptionV3 & IncResV2 &  ResNet50 &  ResNet50 &   ResNet18 & WRN50-2 &  XCiT-M &  XCiT-L \\ \hline
        FGSM & 53.9 & 32.5 & 45.6 & 36.3 & 46.8 & 27.7 & 23.0 & 19.8 \\ 
        BIM & 43.4 & 28.5 & 41.6 & 30.9 & 41.0 & 20.9 & 16.4 & 15.7  \\ 
        MI & 55.9 & 42.5 & 45.7 & 37.4 & 45.7 & 27.8 & 22.8 & 19.8  \\ 
        DI-MI & 61.8 & 52.9 & 47.1 & 38.0 & 47.7 & 31.3 & 25.4 & 21.7 \\ 
        TI-MI & 66.1 & 58.5 & 49.3 & 43.9 & 50.7 & 37.0 & 29.4 & 26.9 \\
        VMI & 72.3 & 66.4 & 51.4 & 47.1 & 48.9 & 36.2 & 33.4 & 30.8 \\
        MI-SVRG & 66.8 & 46.8 & 51.0 & 43.9 & 48.5 & 33.0 & 30.2 & 26.7 \\
        MI-SAM & 64.5 & 47.9 & 50.6 & 43.9 & 48.0 & 33.4 & 31.8 & 26.9 \\
        MI-CSE & 89.6 & 78.2 & 75.0 & 73.5 & 68.4 & 64.4 & 77.5 & 71.0\\
        MI-CWA & 89.6 & 79.1 & 74.6 & 73.6 & 69.5 & 64.8 & 77.8 & 71.7 \\
        MI-ST & \textbf{90.0} & \textbf{81.4} & \textbf{75.8} & \textbf{75.1} & \textbf{69.7} & \textbf{65.1} & \textbf{78.4} & \textbf{72.2} \\
        \hline
    \end{tabular}
    }
    % \caption{\textbf{Black-box attack success rate($\%$,$\uparrow$)}. Our method leads the performance on 8 adversarially trained models available on RobustBench. Note that PGDAT$^{\dag}$ is a variant of PGDAT tuned by bag of tricks. It turns out that our method improve the transferability of the adversarial examples.}
    % \label{tab:t2}
        \end{center}
\end{table*}

Therefore, the update over the entire inner loop is :
\begin{equation*}
    \begin{aligned}
        \bm{x}_t^n-\bm{x}_t^0
        &=\beta \sum_{i=1}^{n}\frac{\bm{g}_i'}{\|\bm{g}_i'\|_2}\\
        &=\beta \sum_{i=1}^n\frac{\bm{g}_i+\beta \bm{H}_i \sum_{j=1}^{i-1}\frac{ \bm{g}_j}{\| \bm{g}_j\|_2} +O(\beta^2)}{||\bm{g}_i + O(\beta)||_2}\\
        &\approx \beta \sum_{i=1}^n\frac{\bm{g}_i+\beta \bm{H}_i \sum_{j=1}^{i-1}\frac{ \bm{g}_j}{\| \bm{g}_j\|_2} +O(\beta^2)}{||\bm{g}_i||_2}\label{eq8}
    \end{aligned}
\end{equation*}

Since:
\begin{equation*}
    \begin{aligned}
        &\mathbb{E}[\frac{\partial}{\partial \bm{x}}\frac{\bm{g}_i\bm{g}_j}{\|\bm{g}_i\|_2\|\bm{g}_j\|_2}] \\
        =&\mathbb{E}[ \frac{\bm{H}_i}{\|\bm{g}_i\|_2}\left(\bm{I}-\frac{\bm{g}_i\bm{g}_i^\top}{\|\bm{g}_i\|_2}\right)\frac{\bm{g}_j}{\|\bm{g}_j\|_2}
        +\frac{\bm{H}_j}{\|\bm{g}_j\|_2}\left(\bm{I}-\frac{\bm{g}_j\bm{g}_j^\top}{\|\bm{g}_j\|_2}\right)\frac{\bm{g}_i}{\|\bm{g}_i\|_2}] \\
        \approx&\mathbb{E}[  \frac{\bm{H}_i}{\|\bm{g}_i\|_2}\frac{\bm{g}_j}{\|\bm{g}_j\|_2}
        +\frac{\bm{H}_j}{\|\bm{g}_j\|_2}\frac{\bm{g}_i}{\|\bm{g}_i\|_2}] \\
        =&2 \mathbb{E}\left[\frac{\bm{H}_i}{\|\bm{g}_i\|_2}\frac{\bm{g}_j}{\|\bm{g}_j\|_2}\right].\label{eq9}
    \end{aligned}
\end{equation*}

We can get :
\begin{equation}
  \begin{aligned}
    &\mathbb{E}[{\bm{x}}_t^n - {\bm{x}}_t^0] \\
    =& \mathbb{E}[\beta \sum_{i=1}^n\frac{\bm{g}_i+\beta \bm{H}_i \sum_{j=1}^{i-1}\frac{ \bm{g}_j}{\| \bm{g}_j\|_2} +O(\beta^2)}{||\bm{g}_i||_2}]\\
    =& \beta \mathbb{E}[\sum_{i=1}^n\frac{\bm{g}_i}{||\bm{g}_i||_2}] + \beta^2 \mathbb{E}[\sum_{j=1}^{i-1}\frac{\bm{H}_i}{\| \bm{g}_i\|_2} \frac{\bm{g}_j}{\|\bm{g}_j\|_2}] + O(\beta^3)\\
    \approx& \beta \mathbb{E}[ \sum_{i=1}^{n} \frac{\bm{g}_i}{\|\bm{g}_i\|_2}]+\frac{\beta^2}{2} \mathbb{E}[\sum_{i,j}^{i<j}\frac{\partial \frac{\bm{g}_i\bm{g}_j}{\|\bm{g}_i\|_2\|\bm{g}_j\|_2}}{\partial \bm{x}}] + O(\beta^3).\label{eq10}
        \end{aligned}
\end{equation}

Hence, our final update, denoted as $g_{whole}$, is:
\begin{equation}
    \begin{aligned}
        &\quad~\bm{g}_{whole}\\ 
        &= \lambda_1 \bm{g}_{mean} + \frac{2\lambda_2}{\beta^2} \bm{g}_{cos} \\
        &=\lambda_1 \mathbb{E}[ \sum_{i=1}^{n} \frac{\bm{g}_i}{\|\bm{g}_i\|_2}] + \frac{2\lambda_2}{\beta^2}[{\bm{x}}_t^{new} - {\bm{x}}_t - \beta \bm{g}_{mean}] \\
        &= \lambda_1\mathbb{E}[ \sum_{i=1}^{n} \frac{\bm{g}_i}{\|\bm{g}_i\|_2}] + \frac{2\lambda_2}{\beta^2}[\frac{\beta^2}{2} \mathbb{E}[\sum_{i,j}^{i<j}\frac{\partial \frac{\bm{g}_i\bm{g}_j}{\|\bm{g}_i\|_2\|\bm{g}_j\|_2}}{\partial \bm{x}}] + O(\beta^3) ] \\
        &={\lambda_1} \mathbb{E}[ \sum_{i=1}^{n} \frac{\bm{g}_i}{\|\bm{g}_i\|_2}] + \lambda_2  \mathbb{E}[\sum_{i,j}^{i<j}\frac{\partial \frac{\bm{g}_i\bm{g}_j}{\|\bm{g}_i\|_2\|\bm{g}_j\|_2}}{\partial \bm{x}}]+ 2\lambda_2O(\beta)\label{eq11}
    \end{aligned}
\end{equation}

Hence, update using our \cref{algo} is equivalent to minimizing:
\begin{equation}
\label{eq:result}
    \max_{\bm{x}} \lambda_1 \frac{1}{|\mathcal{F}_t|}\sum_{f \in \mathcal{F}_t}L(f_i({\bm{x}}), {\bm{y}}) + \lambda_2 \sum_{i,j}^{i<j< |\mathcal{F}_t|}\frac{\bm{g}_i^T\bm{g}_j}{\|\bm{g}_i\|\|\bm{g}_j\|}
\end{equation}

We get the result.
\end{proof}

\begin{remark}
    As shown in \cref{eq:result}, compared with Chen et al.~\cite{chen2023rethinking}, our method has two advantages. First, we could trade off easily between the optimization (original loss) and regularization (the cosine similarity between gradient) by tuning $\lambda_1$ and $\lambda_2$. We do lots of explorations of this trade off in \cref{sec:ablation}. Besides, the error term in our algorithm is $O(\beta)$ rather than $O(\beta^3)$. Since $\beta$ is usually set to a value larger than one, our method incurs much less approximation error.
\end{remark}

%------------------------------------------------------

\subsection{Incorporation with previous attackers}
\label{sec:incorporation}
As shown in \cref{algo}, our method could be view as first calculating the update $\bm{g}_{whole}=\lambda_1 \bm{g}_{loss} + \lambda_2 \bm{g}_{cos}$ and then perform gradient ascent. Since our method is orthogonal to input transformation methods like DI~\cite{xie2019improving}, TI~\cite{dong2019evading}, it can be incorporated with them seamlaessly to achieve improved performance. For gradient-based method, like MI~\cite{dong2018boosting}, NI~\cite{lin2019nesterov} and VMI~\cite{wang2021enhancing}, or state-of-the-art optimizers like Adam~\cite{kingma2014adam}, we could directly view $\bm{g}_{whole}$ as current gradients and calculating the corresponding momentum, as demonstrated in \cref{algo:combine}. Hence, our method is easy  to be combined with other previous works to further improve the transferability.

\begin{algorithm}[bp]
\textbf{Require:} optimization target $\bm{\theta}$, optimizer $\alpha$.

\For{$t$ = $0$ : $T-1$}
{
Calculate $\bm{g}_{whole}$ via \cref{eq5};

Update $\bm{\theta}$ by $\alpha.step(\bm{g}_{whole})$
}
\textbf{Return:} ${\bm{\theta}}_T$
\caption{Combination of ST and other optimizer.}
\label{algo:combine}
\end{algorithm}

\section{Experiment}

%In this chapter, we conducted experiments on our algorithm and provided explanations. We demonstrated the improvements of this algorithm compared to previous ones. On normal models, the attack success rate of this algorithm is consistently no less than $85\%$, with the majority approaching $100\%$. On defense models, the attack success rate is no less than $65\%$, with most rates nearing $80\%$. This underscores the effectiveness and transferability of the adversarial examples generated by our algorithm.

\subsection{Experiment settings}
\label{exp:setting}

\begin{figure*}[t]
	\setlength{\abovecaptionskip}{-5pt}
	\setlength{\belowcaptionskip}{-10pt}
	\begin{minipage}{0.24\linewidth}
		\centering
		\includegraphics[width=\linewidth]{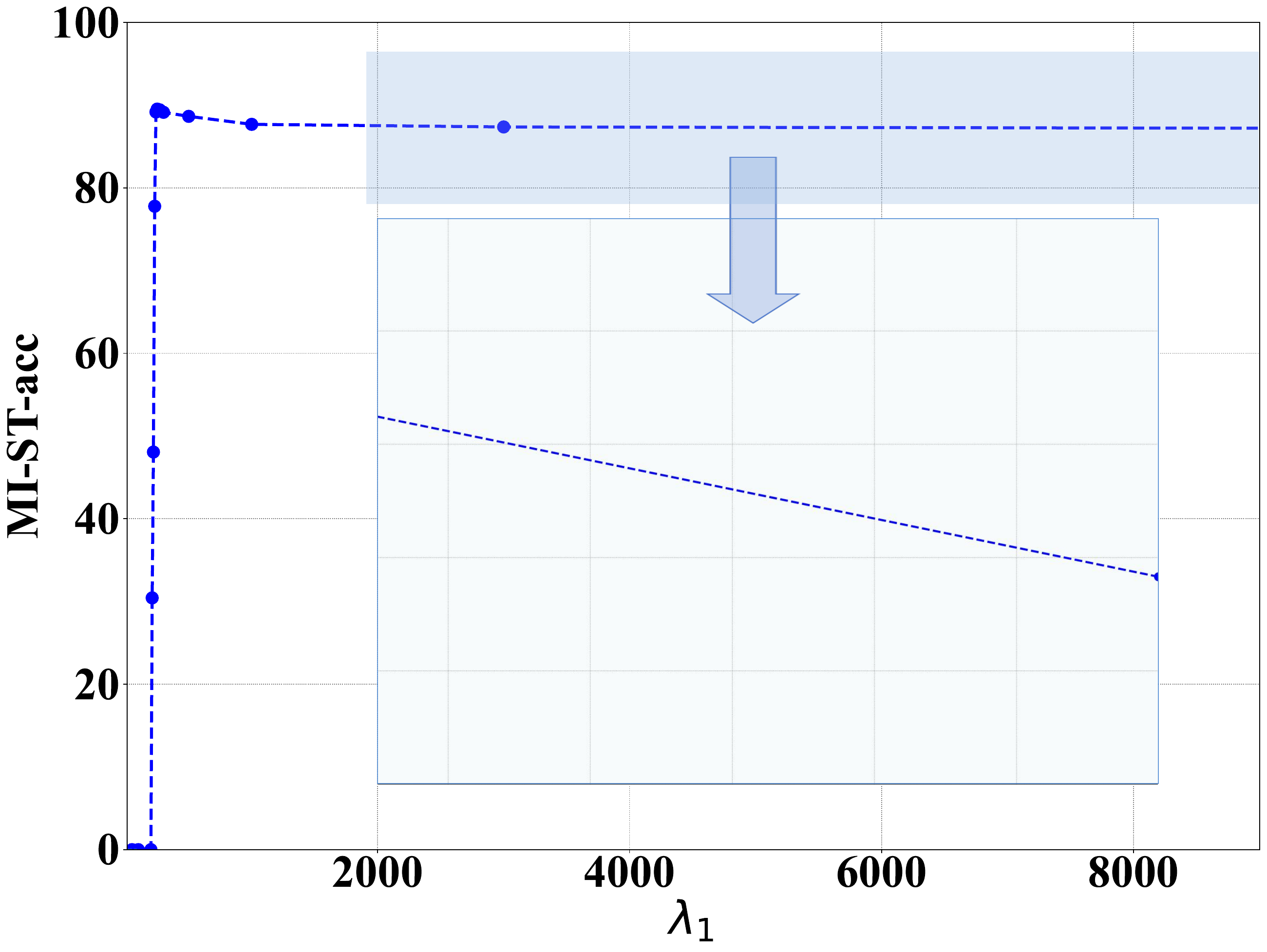}
		\caption{Loss weight $\lambda_1$}
		\label{figure2}
	\end{minipage}
	\hfill
	\begin{minipage}{0.24\linewidth}
		\centering
		\includegraphics[width=\linewidth]{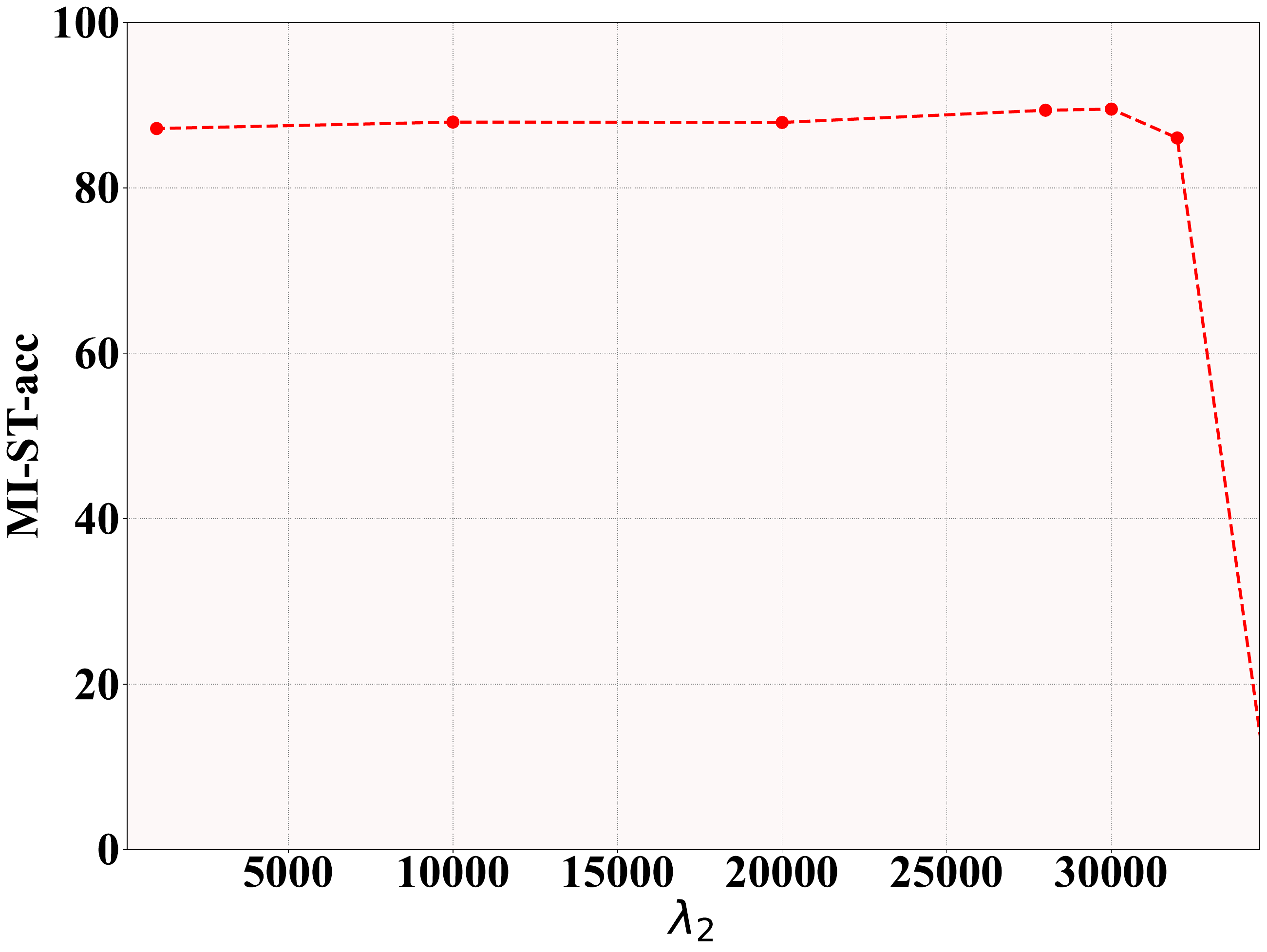}
		\caption{Cosine weight $\lambda_2$}
		\label{figure3}
	\end{minipage}
	\hfill
	\begin{minipage}{0.24\linewidth}
		\centering
		\includegraphics[width=\linewidth]{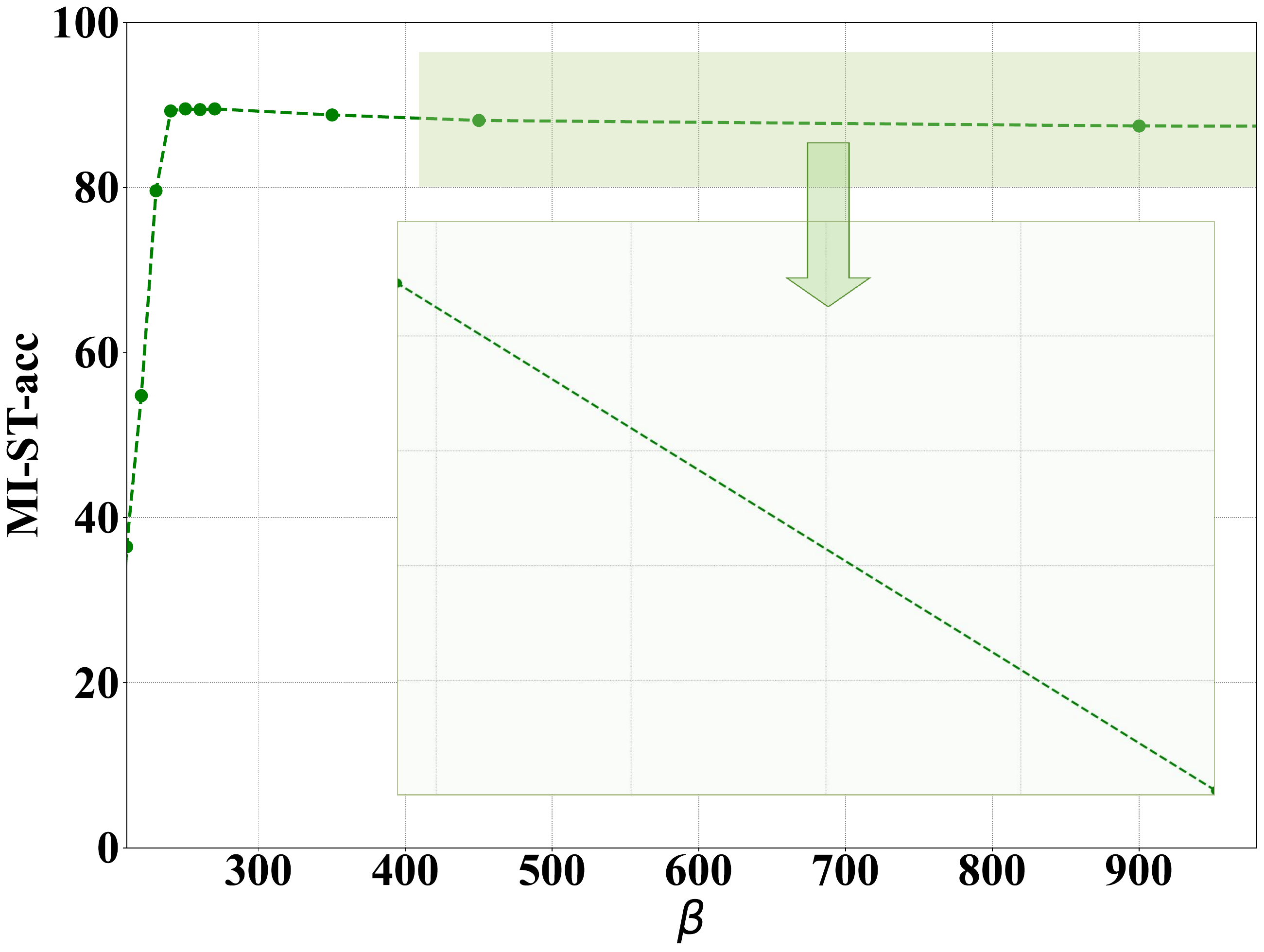}
		\caption{Inner step size $\beta$}
		\label{figure4}
	\end{minipage}
	\hfill
    \begin{minipage}{0.24\linewidth}  
		\centering
		\includegraphics[width=\linewidth]{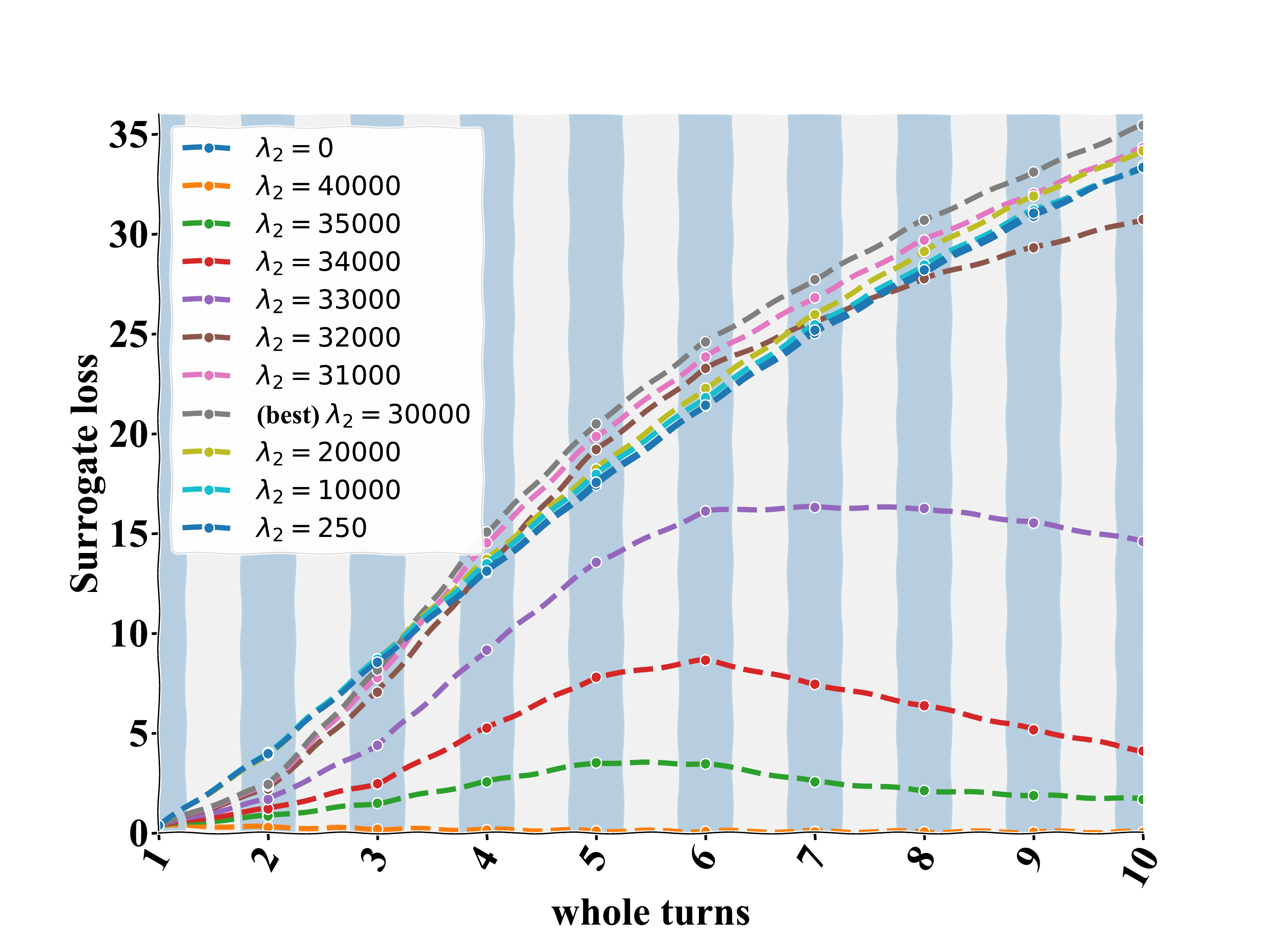}
		\caption{Gap between $\lambda_1$ and $\lambda_2$ }
		\label{figure5}
	\end{minipage}
\end{figure*}

We adopt the same settings as~\cite{chen2023rethinking}.

\textbf{Dataset:}  We use the NIPS2017 dataset, which is comprised of 1000 images selected from ImageNet. All the images are resized to $224\times 224$. 

\textbf{Surrogate Models}: We choose four normally trained models, ResNet18, ResNet32, ResNet50, and ResNet101 from TorchVision~\cite{marcel2010torchvision} and two adversarially trained models, ResNet50~\cite{salman2020adversarially} and XCiT-S12~\cite{debenedetti2023light} from \cite{croce2020robustbench}, which are effective in assessing the method's ability to utilize diverse surrogate models. 

\textbf{Black-box Models:} We evaluate the attack success rate on 24 black-box models, including 16 normally trained models --- AlexNet \cite{alexnet}, VGG-16 \cite{vgg}, GoogleNet \cite{googlenet}, Inception-V3 \cite{inception}, ResNet-152 \cite{resnet}, DenseNet-121 \cite{densenet}, SqueezeNet \cite{iandola2016squeezenet}, ShuffleNet-V2 \cite{ma2018shufflenet}, MobileNet-V3 \cite{mobilenet}, EfficientNet-B0 \cite{tan2019efficientnet}, MNasNet \cite{tan2019mnasnet}, ResNetX-400MF \cite{regnet}, ConvNeXt-T \cite{liu2022convnext}, ViT-B/16 \cite{vit}, Swin-S \cite{liu2021swin}, MaxViT-T \cite{tu2022maxvit}, and 8 adversarially trained models available on RobustBench \cite{croce2020robustbench} --- FGSMAT \cite{adversarialMLAtScale} with Inception-V3, Ensemble AT (EnsAT) \cite{tramer2017ensemble} with Inception-ResNet-V2, FastAT \cite{wong2020fast} with ResNet-50, PGDAT \cite{Engstrom2019Robustness,salman2020adversarially} with ResNet-50, ResNet-18, Wide-ResNet-50-2, a variant of PGDAT tuned by bag-of-tricks (PGDAT$^\dagger$) \cite{debenedetti2022light} with XCiT-M12 and XCiT-L12. Most defense models are state-of-the-art on RobustBench \cite{croce2020robustbench}.

\textbf{Compared Methods:} We compare our methods with FGSM~\cite{goodfellow2014explaining}, BIM~\cite{wang2021adversarial}, MI~\cite{dong2018boosting}, DI~\cite{xie2019improving}, TI~\cite{dong2019evading}, VMI~\cite{wang2021enhancing}, SVRG~\cite{xiong2022stochastic}, SAM~\cite{chen2023rethinking}, CSE~\cite{chen2023rethinking}, CWA~\cite{chen2023rethinking}.  

\textbf{Hyper-parameters:} We set the hyper-parameters as: perturbation threshold $\epsilon = 16/255$, total iteration rounds $T = 10$, momentum decay rate $\mu=1$, learning rate $\beta = 250$, $\alpha = 16/255/5$.

%---------------------------------------
\subsection{Adversarial attacks on state-of-the-art models}
\label{exp:core_result}

\textbf{Attacks on Discriminative Classifiers.} 
The data presented in \cref{tab:t1} clearly demonstrates that our MI-ST algorithm consistently achieves an attack success rate of over 80\% across a variety of state-of-the-art target models, underscoring its efficacy in black-box settings. This high rate of success is observed even against well-established classifiers like VGG16, ResNet152, and DenseNet121, where the MI-ST algorithm reaches near-perfect success rates of 99.7\%, 100.0\%, and 100.0\%, respectively. Additionally, the MI-ST method shows remarkable effectiveness on models that share similarities with any of the surrogate models, leading to significantly higher attack success rates. This can be attributed to the MI-ST's unique approach of encouraging cosine similarity between gradients, which enhances optimization across all surrogate models in tandem. For models not closely resembling the surrogate models, MI-ST still exhibits substantial improvements in attack success rates, often exceeding increases of approximately 20\%. This is particularly evident in its performance on newer architectures like EfficientNetB0 and ViT-B/16, where MI-ST achieves a remarkable success rate of 99.3\% and 90.6\%, respectively. Such results verify the strong generalization ability of our approach, making MI-ST a formidable tool against a wide array of discriminative classifiers.

\textbf{Attacks on Secured Models.} 
The results in~\cref{tab:t2} clearly illustrate the superior performance of the MI-ST algorithm across various adversarially trained target models. Notably, MI-ST consistently outperforms the attack success rates of all existing algorithms, including advanced techniques like MI-CSE and MI-CWA. For instance, on models like InceptionV3 and EnsAT, MI-ST achieves remarkable success rates of 90.0\% and 81.4\%, respectively, outshining other methods by notable margins. This is particularly significant considering that these models represent some of the most robust defenses available on RobustBench. Furthermore, MI-ST demonstrates an improvement of at least 30\% over the MI algorithm in every case, highlighting its enhanced effectiveness. This substantial increase in attack success rate is a testament to the advanced optimization strategies employed by MI-ST. Even against the more recent and sophisticated PGDAT$^{\dag}$ models, MI-ST shows its prowess by achieving success rates of 78.4\% and 72.2\% for the XCiT-M and XCiT-L models, respectively. These results underline the formidable challenge MI-ST poses to even the most secured deep learning models, thereby emphasizing the necessity for developing more robust defensive strategies in the face of potential attacks.

\subsection{Ablation studies}
\label{sec:ablation}

\begin{figure*}
    \centering
    \includegraphics[width=17cm]{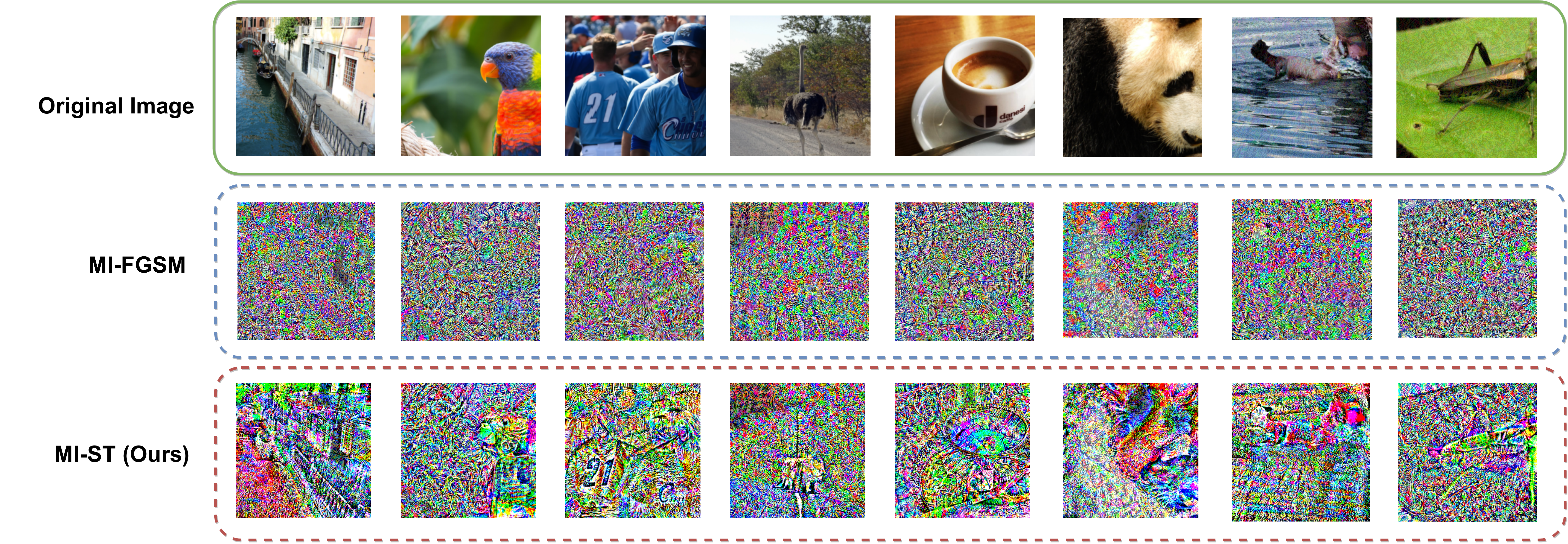}
    \caption{Visualization of adversarial perturbations by MI-FGSM and MI-ST.}
    \label{fig:semantic_gradient}
\end{figure*}

The gradient of the loss, controlled by $\lambda_1$, primarily governs the optimization of the loss function. Meanwhile, the gradient of the cosine term, governed by $\lambda_2$, plays a crucial role in regularization and generalization. Therefore, we ought to carefully tradeoff between these two terms. In the following, we explore the specific roles and effects of these parameters, explaining how they influence regularization and generalization.

\textbf{Loss weight $\lambda_1$:} As shown in \cref{figure2}, an increase in $\lambda_1$ weakens the impact of the regularization, resulting in a slight decline in attack success rate. We also observe that the attack success rate of models that are not similar to the surrogate models drops significantly. This shows the regularization ability of cosine similarity between gradients, especially for models that are not similar to surrogate models. On the other hand, too small a $\lambda_1$ will lead to insufficient optimization of loss functions over the surrogate models, thus leading to a decline in the attack success rate.

\textbf{Cosine weight $\lambda_2$:} To further understand the impact towards optimization, we visualize the average losses over surrogate models concerning $\lambda_2$ in \cref{figure5}. As shown, an excessively large $\lambda_2$ significantly impacts optimization, making it hard to maximize the loss function, resulting in degradation of the attack success rate for both surrogate and target models. Furthermore, as observed in~\cref{figure3}, when $\lambda_2$ gradually decreases, the regularization effect gradually diminishes, leading to a gradual reduction in transferability.

\textbf{Inner step size $\beta$:} Based on~\cref{figure4}, an increase in $\beta$ leads to a larger error term in the Taylor expansion, resulting in a slight decrease in the attack success rate. This decline is more notable for defense models, as attacking such models demands more precise gradients. However, excessively reducing $\beta$ can cause our model to converge into local optima and lead to insufficient optimization, significantly impacting the attack success rate.

\subsection{Discussions}

\textbf{Time complexity. }
Our method maintains the same number of function evaluations (NFEs) as the fundamental methods like MI-FGSM and I-FGSM. In practice, though, our method is slightly slower than MI-FGSM and I-FGSM because it cannot perform backward propagation through all surrogate models simultaneously. In other words, back-propagation for all surrogate models cannot be performed in parallel, as it is necessary to obtain the gradient of each model individually. Despite this, our optimization approach, which ensures the gradients of each surrogate model are similar, leads to a more rapid decrease in loss compared to MI-FGSM. Empirically, this results in our method consuming less time while achieving superior performance compared to MI-FGSM.
% Our method matches MI-FGSM and I-FGSM in terms of the number of function evaluations (NFEs) but is slightly slower in practice. This is because it can't back-propagate through all surrogate models simultaneously; each model's gradient must be obtained individually. However, our optimization approach, which aligns the gradients of each surrogate model, leads to faster loss reduction than MI-FGSM. Empirically, this results in our method consuming less time while achieving superior performance compared to MI-FGSM.

\textbf{Examination in a simple case.} 
In a 2-D simplistic scenario, we further examined the efficacy of our method. As depicted in \cref{fig:illustration}, our approach initially targets regions where the cosine similarity between the gradients of functions \( f \) and \( g \) is high. Consequently, it tends to converge towards a common optimum shared by both models. In contrast, traditional optimizers focus solely on the combined output of \( f+g \), overlooking individual model characteristics. This approach often leads to convergence at the minimizer of \( f+g \) alone, potentially compromising the effectiveness of attacks on \( f \) individually.

\textbf{Visualization of Adversaries.} 
We present visualizations of the adversarial perturbations for the first nine images in the dataset, using the same experimental setup as in \cref{exp:setting} and illustrated in \cref{fig:semantic_gradient}. The perturbations generated by our method demonstrate greater semantic meaning compared to those crafted by MI-FGSM. This indicates that our attack strategy is more adept at targeting and obscuring key features in the images. Such targeted perturbations are likely to exploit common vulnerabilities across different models, thereby improving the transferability of the attack to various unseen threat models.

\section{Conclusion}
In this paper, we propose a Similar  Target method to make full use of the information in each surrogate models by promoting cosine similarity between the gradients so as to attack all the models simultaneously. We conduct extensive experiments to validate the effectiveness of the proposed methods and explain why they work in practice. To further improve the transferability of the generated adversarial examples, we carried out some experiments focusing on finding the best trade-off between optimization and regularization. Among 24 discriminative classifiers and defense models, our method outperforms state-of-the-art attackers on 18 of them. The results demonstrate the vulnerability of the existing defenses and raise security issues for the development of more robust deep learning models.

\bibliographystyle{IEEEbib}
\bibliography{IEEEabrv,refs}
\end{document}